\documentclass[preprint,12pt]{elsarticle}

\usepackage{amsmath,amssymb,amsfonts,amsthm,mathtools,bm}
\usepackage{booktabs,multirow,array}
\usepackage{graphicx}
\usepackage{tikz,pgfplots}
\usetikzlibrary{arrows.meta,positioning}
\usepgfplotslibrary{groupplots}
\pgfplotsset{compat=1.18}
\usepackage{xcolor}
\usepackage{enumitem}
\usepackage{microtype}
\usepackage{url}
\usepackage{etoolbox}

\AtBeginEnvironment{table}{%
    \setlength{\belowcaptionskip}{6pt}%
}

\journal{Neurocomputing}

\newtheorem{theorem}{Theorem}
\newtheorem{lemma}{Lemma}
\newtheorem{proposition}{Proposition}
\newtheorem{corollary}{Corollary}

\newcommand{\R}{\mathbb{R}}
\newcommand{\cW}{\mathcal{W}}
\newcommand{\cX}{\mathcal{X}}
\newcommand{\cE}{\mathcal{E}}
\newcommand{\cG}{\mathcal{G}}
\newcommand{\cD}{\mathcal{D}}
\newcommand{\cH}{\mathcal{H}}
\newcommand{\br}{\bm{r}}

\newcommand{\meanSD}[2]{\shortstack{$#1$\\[-1pt]{\scriptsize$(#2)$}}}
\newcommand{\bestSD}[2]{\shortstack{$\mathbf{#1}$\\[-1pt]{\scriptsize$(#2)$}}}
\usepackage{float}
\floatstyle{ruled}
\newfloat{algorithm}{tbp}{loa}
\floatname{algorithm}{Algorithm}

\begin{document}

\begin{frontmatter}

\title{Residual-Guided Adaptive Tucker Rank Allocation for Parameter-Efficient MCI Conversion Prediction}

\author[inst1,inst2]{Guanghua He}		

\author[inst2]{Hancan Zhu\corref{cor1}}
\ead{hancanzhu@yeah.net}

\author[inst3]{Gaohang Yu\corref{cor2}}
\ead{maghyu@163.com}

\author[inst1]{An Zhang}

\author{for~the~Alzheimer's Disease Neuroimaging Initiative\fnref{adni}}
\fntext[adni]{Data used in preparation of this article were obtained from
    the Alzheimer's Disease Neuroimaging Initiative (ADNI) database
    (\url{https://adni.loni.usc.edu}). As such, the investigators within
    the ADNI contributed to the design and implementation of ADNI and/or
    provided data but did not participate in the analysis or writing of
    this report. A complete listing of ADNI investigators can be found at:
    \url{https://adni.loni.usc.edu/wp-content/uploads/how_to_apply/ADNI_Acknowledgement_List.pdf}.}

\cortext[cor1]{Corresponding author.}
\cortext[cor2]{Corresponding author.}

\affiliation[inst1]{
	organization={Department of Mathematics, Hangzhou Dianzi University},
	city={Hangzhou},
	postcode={310018},
	country={China}
}

\affiliation[inst2]{
	organization={School of Mathematics, Physics and Information, Shaoxing University},
	city={Shaoxing},
	postcode={312000},
	country={China}
}

\affiliation[inst3]{
	organization={Department of Mathematics, Zhejiang University of Science and Technology},
	city={Hangzhou},
	postcode={310023},
	country={China}
}

\begin{abstract}
	Allocating limited trainable capacity across convolutional
	layers is a central challenge in parameter-efficient
	fine-tuning (PEFT) of pretrained three-dimensional (3D) networks.
	We propose Residual-Guided Adaptive Tucker Rank Allocation,
	which jointly selects input- and output-channel ranks
	under a fixed adapter-parameter budget while preserving
	the spatial kernel modes.
	To guide rank selection, we learn task-specific proxy updates
	from downstream training data and approximate them using
	candidate Tucker rank pairs.
	A perturbation analysis bounds the output discrepancy between
	the fixed proxy network and its Tucker approximation
	in terms of layerwise reconstruction residuals,
	motivating residual-based allocation.
	We normalize each squared residual by the corresponding
	proxy-update energy to compare relative approximation
	errors across layers.
	Dynamic programming minimizes their sum under a global
	budget constraint that accounts for both channel factors
	and the Tucker core.
	This yields a globally optimal allocation over the specified
	candidate rank pairs.
	We evaluate the method on the Alzheimer's Disease Neuroimaging
	Initiative (ADNI) cohort for predicting conversion from
	mild cognitive impairment to Alzheimer's disease within
	36 months, using hippocampal segmentation pretraining.
	Across five random seeds, the proposed method achieves a mean area
	under the receiver operating characteristic curve (AUC) of $0.7706$,
	compared with $0.7239$ for fixed-ratio Tucker adaptation
	and $0.7430$ for full fine-tuning.
	The total number of trainable parameters is limited to approximately
	$0.618$ million, accounting for only $1.45\%$ of the full model.
	Adaptive Tucker rank allocation thus provides an effective approach
	to parameter-efficient transfer learning in 3D medical imaging.
\end{abstract}


\end{frontmatter}

\section{Introduction}

Adapting pretrained three-dimensional (3D) networks to small
labeled cohorts is an important problem in medical image analysis.
Volumetric pretraining can provide useful anatomical representations,
as demonstrated by Med3D and Models Genesis
\cite{chen2019med3d,zhou2019genesis}.
However, these representations must be adapted when the downstream
objective differs from the pretraining task.
For example, hippocampal segmentation provides local anatomical
supervision, whereas predicting conversion from mild cognitive
impairment (MCI) to Alzheimer's disease (AD) requires
subject-level information relevant to a future clinical outcome
\cite{liu2020hippocampus,spasov2019conversion}.
The challenge is to learn progression-related representations
from limited labeled data while updating only a small
fraction of the pretrained network's parameters.

Parameter-efficient fine-tuning (PEFT) reduces the number of
trainable parameters by retaining a largely frozen backbone
and optimizing a small set of task-specific parameters.
Bottleneck adapters and Low-Rank Adaptation (LoRA) are
representative approaches \cite{houlsby2019adapters,hu2022lora},
with medical-imaging studies demonstrating their applicability
to data-limited transfer \cite{dutt2024missed,shen2024medtuning}.
Beyond choosing an adapter parameterization, a key design
question is how to distribute trainable capacity across layers.
AdaLoRA, Explained Variance Adaptation (EVA), and GoRA
address this question through importance scores, activation
statistics, or gradients that guide matrix-rank allocation
\cite{zhang2023adalora,paischer2025eva,he2025gora}.
These approaches incorporate task information into capacity
allocation, rather than determining layerwise ranks solely
from parameter dimensions.

For 3D convolutional networks, capacity allocation must also
account for the tensor structure of the adapted kernels.
Each kernel has distinct input-channel, output-channel,
and spatial modes.
Matrix unfolding preserves its entries, but a single
matrix-rank constraint does not separately specify
the ranks of the two channel modes.
Tensor-based PEFT methods, including FLoRA, offer
parameterizations that retain multidimensional update
structure \cite{si2024flora}.
In channel-mode Tucker adaptation, input- and output-channel
ranks can be selected independently while the spatial
kernel dimensions remain unchanged.
However, the product of these ranks determines the
channel dimensions of the shared core, coupling their
parameter costs.
Rank allocation therefore requires evaluating the
approximation quality and cost of each rank pair
under a global budget.
Fixed-ratio allocation sets both ranks from channel
dimensions, without assessing how well the resulting
capacity represents a downstream task-specific update.

We address this limitation with Residual-Guided Adaptive
Tucker Rank Allocation, hereafter referred to as
residual-guided Tucker. The method uses normalized reconstruction residuals of task-specific proxy updates learned from downstream training
data to guide joint input- and output-channel rank selection.
A perturbation analysis bounds the output discrepancy between
the fixed proxy network and its Tucker approximation
in terms of layerwise reconstruction errors,
providing theoretical support for residual-based allocation.
Dynamic programming minimizes the sum of normalized residuals
under a fixed adapter-parameter budget.
We evaluate the method by transferring an encoder pretrained
for hippocampal segmentation to 36-month MCI conversion
prediction on the Alzheimer's Disease Neuroimaging Initiative
(ADNI) cohort \cite{jack2008adni}.
Comparisons with fixed-ratio Tucker and other PEFT baselines,
together with ablation studies, assess predictive performance,
parameter efficiency, and the role of the allocation criterion.

The main contributions are as follows:
\begin{enumerate}[leftmargin=2em]
	\item \textbf{Task-conditioned Tucker rank allocation.}
	We formulate joint input- and output-channel rank selection
	for additive 3D Tucker adapters under a global parameter budget.
	Normalized reconstruction residuals of downstream proxy updates
	guide allocation, with exact parameter costs accounting for
	both channel factors and their shared core.
	
	\item \textbf{Theoretical support and exact optimization.}
	We relate layerwise reconstruction errors to an
	output-perturbation bound for the fixed proxy network,
	providing a theoretical basis for residual-based allocation.
	Dynamic programming yields a globally optimal solution
	over the specified discrete candidate rank pairs.
	
	\item \textbf{Evaluation of segmentation-to-progression transfer.}
	We evaluate the method on ADNI for MCI conversion prediction
	using hippocampal segmentation pretraining.
	Comparisons and ablations examine the effects of rank allocation
	and residual normalization. On the evaluated split, residual-guided
	Tucker achieves a higher mean area under the receiver operating
	characteristic curve (AUC) than fixed-ratio Tucker under the same
	adapter-budget ceiling.
\end{enumerate}

\section{Related Work}

\subsection{MRI-based progression prediction and medical transfer learning}

Magnetic resonance imaging (MRI) research on Alzheimer's disease encompasses
diagnostic classification \cite{wen2020reproducible} and the
prediction of future clinical progression. Basaia et al.\ investigate AD classification and MCI conversion prediction from a single MRI \cite{basaia2019classification},
whereas Spasov et al.\ combine imaging and non-imaging inputs
in a compact multitask model for conversion prediction
\cite{spasov2019conversion}.
Alongside these predictive models, anatomical supervision
has been explored for learning disease-related representations.
Liu et al.\ integrate hippocampal segmentation with
AD-related classification, connecting regional anatomical
features with subject-level diagnostic prediction
\cite{liu2020hippocampus}.

Volumetric pretraining provides representations that can
be transferred to downstream medical imaging tasks.
Med3D learns from segmentation supervision, while Models
Genesis uses self-supervised restoration to learn
transferable 3D features
\cite{chen2019med3d,zhou2019genesis}.
Building on pretrained representations, medical PEFT studies
investigate adaptation with a limited number of trainable
parameters \cite{dutt2024missed,lian2024less}.
Med-Tuning extends this direction to volumetric segmentation
through an efficient tuning framework
\cite{shen2024medtuning}.

\subsection{Low-rank adaptation, initialization, and rank allocation}

Evidence that pretrained language models can be fine-tuned
within low-dimensional parameter subspaces supports the
development of compact adaptation methods
\cite{aghajanyan2021intrinsic}.
Bottleneck adapters introduce trainable feature transformations,
whereas LoRA represents additive weight updates as products
of two low-rank factors \cite{houlsby2019adapters,hu2022lora}.
Subsequent work refines the parameterization and initialization
of these updates.
DoRA separates weight magnitude from direction
\cite{liu2024dora}, while PoLAR introduces a polar-inspired
representation and optimization scheme \cite{lion2025polar}.
PiSSA initializes adaptation from principal singular components
of pretrained weights \cite{meng2024pissa}.
LoRA-GA instead uses downstream gradients to align the initial
low-rank update direction with the full-fine-tuning gradient
\cite{wang2024loraga}.

Rank allocation determines how adaptation capacity is
distributed across layers under a parameter budget.
AdaLoRA adjusts ranks by pruning singular components
according to estimated importance \cite{zhang2023adalora}.
EVA combines activation-based initialization with rank
redistribution guided by explained variance
\cite{paischer2025eva}, while GoRA uses gradient information
to guide both rank allocation and initialization
\cite{he2025gora}.
These methods incorporate information from model parameters,
activations, or gradients into the selection of
layerwise adaptation capacity.

\subsection{Tensor-structured adaptation and budgeted approximation}

Tensor-structured adapters use factorization and parameter
sharing to construct compact trainable updates.
Compacter employs Kronecker-based hypercomplex layers,
while LoRETTA uses tensor-train parameterizations
\cite{mahabadi2021compacter,yang2024loretta}.
FacT tensorizes vision-transformer weights and factorizes
their increments into lightweight components
\cite{jie2023fact}.
FLoRA preserves the multidimensional structure of weight
updates \cite{si2024flora}, whereas TuckA combines Tucker
experts with hierarchical routing \cite{lei2026tucka}.
Applications to medical image segmentation include
tensor CUR adaptation in tCURLoRA \cite{he2026tcurlora}
and principal tensor singular-component tuning in
LoRA-PT \cite{he2025lorapt}.

Tensor compression provides approximation techniques
and optimization formulations relevant to rank allocation.
Tucker decomposition and higher-order singular value
decomposition (HOSVD) provide low-rank representations
along individual tensor modes
\cite{kolda2009tensor,deLathauwer2000}.
Kim et al.\ apply Tucker decomposition to convolutional
kernels \cite{kim2016compression}, while ADA-Tucker explores
adaptable tensorization through dimension adjustment
\cite{zhong2019adatucker}.
BATUDE incorporates Tucker rank selection into budget-constrained
compression \cite{yin2022batude}.
ROCKET combines reconstruction-error criteria with a
multiple-choice knapsack formulation and dynamic programming
for layerwise compression \cite{ali2026rocket}.

Our method formulates task-conditioned Tucker rank allocation
as a budget-constrained approximation problem for downstream
proxy updates. Normalized reconstruction residuals quantify the relative
approximation error of each candidate rank pair,
while exact parameter costs capture the coupling between
the two channel ranks. This formulation connects task-specific update structure with joint rank selection for 3D convolutional adapters,
while preserving the spatial kernel dimensions.

\section{Method}
\label{sec:method}

\subsection{Overview and notation}

Figure~\ref{fig:framework} outlines the three stages of
residual-guided Tucker.
First, calibration using only training data estimates
task-specific proxy updates for the selected convolutional kernels.
Second, candidate Tucker rank pairs are evaluated by normalized
reconstruction residuals and jointly selected under a global
adapter-parameter budget.
Finally, PEFT training starts from the original pretrained
checkpoint with newly initialized adapters at the selected ranks.

Let $\ell\in\{1,\ldots,L\}$ index the adapted layers and
$q$ index the candidate rank pairs within each layer.
The rank pair for layer $\ell$ is
$\br_\ell=(r_{\mathrm{out},\ell},r_{\mathrm{in},\ell})$,
and $\br=(\br_1,\ldots,\br_L)$ denotes the network-wide
rank configuration.
For the pretrained kernel $\cW_\ell$, we distinguish
the trainable adapter update $\Delta\cW_\ell$ from
the calibrated proxy update $\Delta\cW_\ell^{\mathrm{proxy}}$.
Both have the same dimensions as $\cW_\ell$, but the proxy
update is held fixed during rank selection and serves
as the approximation target in the analysis.
We denote the tensor Frobenius norm by $\|\cdot\|_F$,
the Euclidean norm of vectorized features or logits
by $\|\cdot\|_2$, and the operator norm induced by
the Euclidean norm by $\|\cdot\|_{\mathrm{op}}$.

\begin{table}[t]
	\centering
	\small
	\caption{Main notation for residual-guided Tucker allocation.}
	\label{tab:notation}
	\begin{tabular}{lp{0.66\linewidth}}
		\toprule
		Symbol & Meaning\\
		\midrule
		$\cW_\ell,\Delta\cW_\ell$
		& Pretrained kernel and trainable adapter update at layer $\ell$\\
		
		$\Delta\cW_\ell^{\mathrm{proxy}}$
		& Calibration-induced kernel update, held fixed during rank allocation\\
		
		$\br_\ell,\br_{\ell,q}$
		& Channel-rank pair at layer $\ell$ and its $q$th candidate pair\\
		
		$\br,\mathcal R$
		& Network-wide rank configuration and Cartesian product of
		layerwise candidate sets\\
		
		$\cH_{\boldsymbol{s}}(\cX)$
		& Channel-mode HOSVD approximation of $\cX$ at rank pair
		$\boldsymbol{s}$\\
		
		$\cE_\ell(\br_\ell),\cE_{\ell,q}$
		& Proxy-update reconstruction residual tensor and its
		candidate-indexed form\\
		
		$Z_\ell$
		& Proxy-update energy,
		$\|\Delta\cW_\ell^{\mathrm{proxy}}\|_F^2$\\
		
		$d_{\ell,q}$
		& Normalized squared reconstruction residual,
		$\|\cE_{\ell,q}\|_F^2/\max(Z_\ell,\varepsilon)$\\
		
		$S_{\ell,q}$
		& Candidate allocation score; $S_{\ell,q}=d_{\ell,q}$
		for the proposed method\\
		
		$p_{\ell,q},B_{\mathrm{adp}}$
		& Candidate adapter-parameter cost and total adapter-budget ceiling\\
		
		$\kappa_\ell(x),\bar\kappa_\ell$
		& Layerwise output-perturbation coefficient and its
		root-mean-square value over inputs\\
		
		$C$
		& Rank-independent coefficient in the normalized-residual
		perturbation bound\\
		\bottomrule
	\end{tabular}
\end{table}

\begin{figure}[t]
	\centering
	\includegraphics[width=\linewidth]{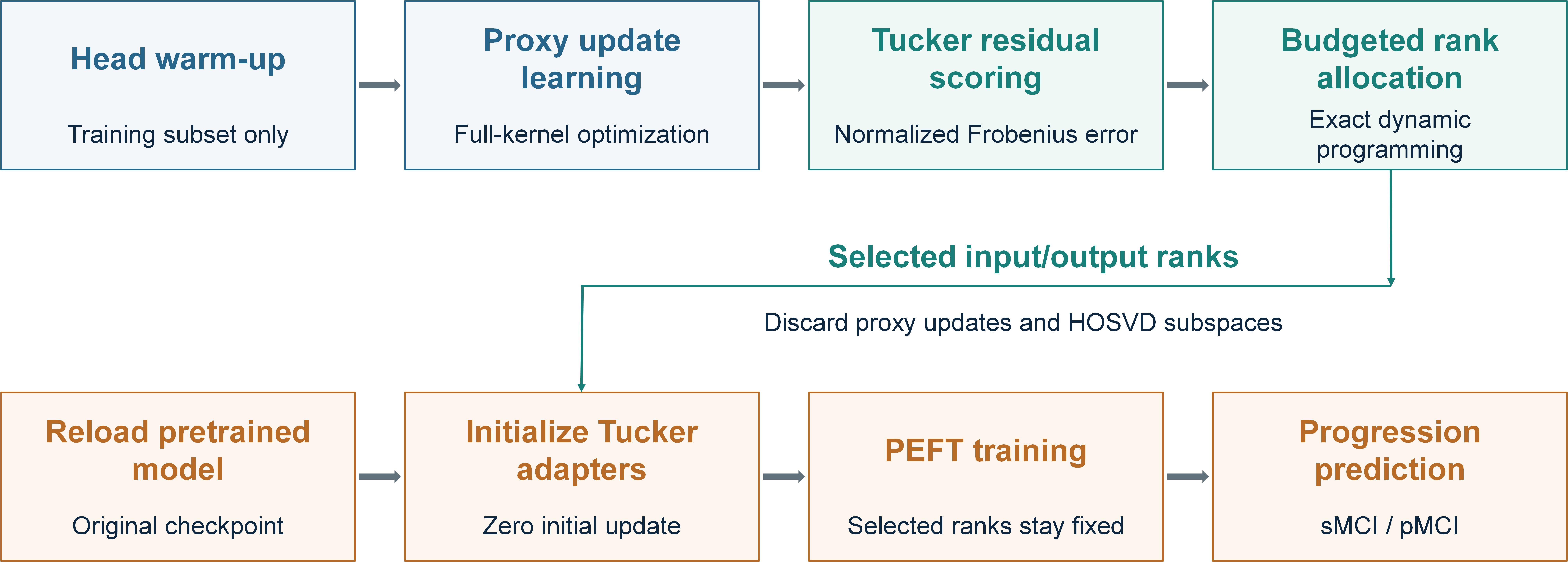}
	\caption{Overview of residual-guided Tucker.
		Task-specific proxy updates provide normalized reconstruction
		residuals for joint input- and output-channel rank allocation
		under a fixed adapter-parameter budget.
		Final PEFT training starts from the original pretrained
		checkpoint with adapters initialized to produce zero updates
		at the selected ranks.}
	\label{fig:framework}
\end{figure}

\subsection{Channel-mode Tucker adaptation and parameter budget}
\label{subsec:tucker-parameterization}

For each adapted layer $\ell$, the pretrained convolutional kernel is
\[
\cW_\ell\in
\R^{C_{\mathrm{out},\ell}\times C_{\mathrm{in},\ell}
	\times k_{\ell,1}\times k_{\ell,2}\times k_{\ell,3}},
\]
with spatial volume $K_\ell=k_{\ell,1}k_{\ell,2}k_{\ell,3}$.
We parameterize its additive update using a channel-mode Tucker
decomposition, retaining the original spatial dimensions:
\begin{equation}
	\Delta\cW_\ell
	=
	\cG_\ell
	\times_1 U_{\mathrm{out},\ell}
	\times_2 U_{\mathrm{in},\ell},
	\label{eq:tucker-update}
\end{equation}
where $\times_n$ denotes the mode-$n$ tensor--matrix product,
\[
U_{\mathrm{out},\ell}
\in\R^{C_{\mathrm{out},\ell}\times r_{\mathrm{out},\ell}},
\qquad
U_{\mathrm{in},\ell}
\in\R^{C_{\mathrm{in},\ell}\times r_{\mathrm{in},\ell}},
\]
and
\[
\cG_\ell\in
\R^{r_{\mathrm{out},\ell}\times r_{\mathrm{in},\ell}
	\times k_{\ell,1}\times k_{\ell,2}\times k_{\ell,3}}.
\]
The adapted kernel is $\cW_\ell+\Delta\cW_\ell$, with
$\cW_\ell$ frozen during PEFT.
The integer channel ranks satisfy
$1\le r_{\mathrm{out},\ell}\le C_{\mathrm{out},\ell}$ and
$1\le r_{\mathrm{in},\ell}\le C_{\mathrm{in},\ell}$.
They determine the factor widths and bound the corresponding
mode-unfolding ranks of the update, which may be lower.

The two channel factors and the Tucker core contain
\begin{equation}
	p_\ell(\br_\ell)
	=
	C_{\mathrm{out},\ell}r_{\mathrm{out},\ell}
	+
	C_{\mathrm{in},\ell}r_{\mathrm{in},\ell}
	+
	K_\ell r_{\mathrm{out},\ell}r_{\mathrm{in},\ell}
	\label{eq:tucker-cost}
\end{equation}
trainable parameters.
This expression counts all stored factor and core entries,
rather than the intrinsic degrees of freedom of the represented
update.
The bilinear core term couples the input- and output-channel
ranks, so their parameter costs must be considered jointly.
Across the adapted layers, rank selection is subject to
\begin{equation}
	\sum_{\ell=1}^{L}p_\ell(\br_\ell)\le B_{\mathrm{adp}},
	\label{eq:adapter-budget}
\end{equation}
where $B_{\mathrm{adp}}$ is the total adapter-parameter budget.

The total trainable parameter count additionally includes
the common components trained by all compared adapter methods.
Their contribution is fixed across these methods and is
specified in Section~\ref{subsec:network-training}.

\subsection{Proxy calibration and residual scoring}
\label{subsec:calibration}

We draw a class-balanced calibration subset $\cD_{\mathrm{cal}}$ from the downstream training partition. The classification head is first warmed up with the encoder frozen, allowing it to adapt to the pretrained representation before kernel calibration.

We then hold all other model parameters fixed and optimize the selected full convolutional kernels for $T_{\mathrm p}$ steps. The resulting displacement for layer $\ell$ is
\begin{equation}
\Delta\cW_\ell^{\mathrm{proxy}}
=\cW_\ell^{(T_{\mathrm p})}-\cW_\ell^{(0)}.
\label{eq:proxy-update}
\end{equation}
Here $\cW_\ell^{(t)}$ denotes the kernel after $t$ calibration steps, with $\cW_\ell^{(0)}=\cW_\ell$. The displacements provide task-conditioned references for rank scoring, without assuming convergence to optimal downstream updates. They are held fixed during candidate evaluation. Calibration settings are given in Section~\ref{subsec:calibration-settings}.

\paragraph{Channel-mode approximation}
\label{subsec:tucker-geometry}

To evaluate candidate ranks, we approximate each proxy update
along its two channel modes.
Let $\cX$ denote the proxy update
$\Delta\cW_\ell^{\mathrm{proxy}}$ at a given layer.
Choose orthonormal left-singular bases of its output- and
input-channel unfoldings, ordered by nonincreasing singular values.
Let $P_{\mathrm{out}}(r_{\mathrm{out}};\cX)$ and
$P_{\mathrm{in}}(r_{\mathrm{in}};\cX)$ denote the orthogonal
projectors onto the leading $r_{\mathrm{out}}$ and
$r_{\mathrm{in}}$ basis vectors, respectively.
Within subspaces associated with repeated singular values,
any orthonormal basis may be chosen, provided that it is
held fixed across candidate ranks.
For a rank pair
$\boldsymbol{s}=(r_{\mathrm{out}},r_{\mathrm{in}})$,
the channel-mode HOSVD approximation
\cite{kolda2009tensor,deLathauwer2000} is
\begin{equation}
	\cH_{\boldsymbol{s}}(\cX)
	=
	\cX\times_1P_{\mathrm{out}}(r_{\mathrm{out}};\cX)
	\times_2P_{\mathrm{in}}(r_{\mathrm{in}};\cX).
	\label{eq:two-mode-hosvd}
\end{equation}
For fixed $\cX$ and $\boldsymbol{s}$, the two mode projections
commute, and their composition is an orthogonal projection
with respect to the Frobenius inner product.

\paragraph{Normalized residual score}
For a rank pair $\br_\ell$, define
\begin{equation}
\cE_\ell(\br_\ell)
=\Delta\cW_\ell^{\mathrm{proxy}}
-\cH_{\br_\ell}(\Delta\cW_\ell^{\mathrm{proxy}}),
\qquad Z_\ell=\|\Delta\cW_\ell^{\mathrm{proxy}}\|_F^2.
\label{eq:global-residual}
\end{equation}
The normalized residual is
\begin{equation}
d_\ell(\br_\ell)
=\frac{\|\cE_\ell(\br_\ell)\|_F^2}{\max(Z_\ell,\varepsilon)},
\qquad \varepsilon>0.
\label{eq:relative-residual}
\end{equation}
Orthogonality gives $\|\cE_\ell(\br_\ell)\|_F^2=Z_\ell-\|\cH_{\br_\ell}(\Delta\cW_\ell^{\mathrm{proxy}})\|_F^2$, and hence $0\le d_\ell\le1$. For $Z_\ell\ge\varepsilon$, the score is exactly the fraction of discarded proxy energy. Here $\varepsilon$ is a small numerical floor. Below this floor the score is a stabilized residual, with value zero for a zero update. Normalization removes the direct energy-scale preference of an absolute residual; it does not estimate predictive importance or the fraction of task performance lost.

For candidate $\br_{\ell,q}$, we write $\cE_{\ell,q}=\cE_\ell(\br_{\ell,q})$ and $d_{\ell,q}=d_\ell(\br_{\ell,q})$. The allocation score is
\begin{equation}
S_{\ell,q}=d_{\ell,q}.
\label{eq:allocation-score}
\end{equation}
These scores compare candidate approximations on a common relative scale; their connection to network perturbation is established next.

\subsection{Approximation analysis}
\label{subsec:global-sensitivity}

Although channel-mode HOSVD does not necessarily yield the
best approximation at prescribed ranks, its reconstruction
error admits a bound in terms of the unfolding singular values.

\begin{lemma}[Channel-truncated HOSVD residual]
	\label{lem:hosvd-bound}
	Let $\{\sigma_{\mathrm{out},i}\}$ and
	$\{\sigma_{\mathrm{in},j}\}$ denote the singular values
	of the output- and input-channel unfoldings of $\cX$.
	Then
	\begin{equation}
		\left\|
		\cX-\cH_{(r_{\mathrm{out}},r_{\mathrm{in}})}(\cX)
		\right\|_F^2
		\le
		\sum_{i>r_{\mathrm{out}}}\sigma_{\mathrm{out},i}^2
		+
		\sum_{j>r_{\mathrm{in}}}\sigma_{\mathrm{in},j}^2.
		\label{eq:hosvd-tail-bound}
	\end{equation}
\end{lemma}

The proof is given in Appendix~\ref{app:hosvd}.
The proposed allocator uses the actual HOSVD reconstruction
residual, while Eq.~\eqref{eq:hosvd-tail-bound} provides
its approximation-theoretic interpretation.

Consider a fixed calibration run with finite candidate sets
$\mathcal R_\ell$ and
$\mathcal R=\prod_{\ell=1}^{L}\mathcal R_\ell$.
Let $f^{\mathrm{proxy}}$ denote the evaluation-mode network
with kernels
$\cW_\ell+\Delta\cW_\ell^{\mathrm{proxy}}$.
For $\br\in\mathcal R$, let $\widehat f_{\br}$ denote
the network obtained by replacing each proxy update with
$\cH_{\br_\ell}(\Delta\cW_\ell^{\mathrm{proxy}})$.
Both networks share all other parameters, with fixed batch
normalization (BatchNorm) statistics and disabled dropout.
The truncated proxy network is used only for analysis and
is distinct from the model used for final PEFT training.

For each $\br\in\mathcal R$, consider intermediate networks
formed by replacing the adapted nodes one at a time in
topological order.
Assume that the common input feature at node $\ell$ has norm
at most $B_\ell(x)$ and that the downstream mapping is
$D_\ell$-Lipschitz with respect to the node output.
Let $\mathcal A_\ell[\cE]$ denote the convolution operator
induced by a kernel perturbation $\cE$.
For fixed input dimensions and convolution settings,
this mapping is linear and bounded, so a constant
$\beta_\ell$ exists such that
\begin{equation}
	\|\mathcal A_\ell[\cE]\|_{\mathrm{op}}
	\le\beta_\ell\|\cE\|_F.
	\label{eq:kernel-operator-compact}
\end{equation}

Define
\begin{equation}
	\kappa_\ell(x)=D_\ell B_\ell(x)\beta_\ell,
	\qquad
	\bar\kappa_\ell=
	\bigl(\mathbb E\kappa_\ell(X)^2\bigr)^{1/2}.
	\label{eq:compact-sensitivity}
\end{equation}
Here the expectation is over input volumes $X$,
conditional on the fixed calibration.
Because the candidate sets are finite, uniform bounds can
be established over the intermediate networks using the
Lipschitz properties of their constituent operations.
Residual connections are included in the downstream mapping,
with unaffected bypass inputs held fixed.
A finite second moment of $X$ ensures
$\bar\kappa_\ell<\infty$.
These bounds are independent of the selected rank configuration
and are not used as allocation weights.

\begin{theorem}[Output perturbation under Tucker approximation]
	\label{thm:task-sensitive-tucker}
	For a fixed calibrated proxy network, suppose that the
	input-feature bounds $B_\ell(x)$, downstream Lipschitz
	constants $D_\ell$, and convolution-operator bounds
	$\beta_\ell$ hold uniformly over all intermediate-network
	comparisons associated with $\br\in\mathcal R$.
	Then, for every $\br\in\mathcal R$,
	\begin{equation}
		\|f^{\mathrm{proxy}}(x)-\widehat f_{\br}(x)\|_2
		\le
		\sum_{\ell=1}^{L}
		\kappa_\ell(x)\|\cE_\ell(\br_\ell)\|_F.
		\label{eq:compact-output-bound}
	\end{equation}
	If $\bar\kappa_\ell<\infty$ for all $\ell$, then
	\begin{align}
		\mathbb E\|f^{\mathrm{proxy}}(X)-\widehat f_{\br}(X)\|_2^2
		&\le
		\left(
		\sum_{\ell=1}^{L}
		\bar\kappa_\ell\|\cE_\ell(\br_\ell)\|_F
		\right)^2
		\nonumber\\
		&\le
		L\sum_{\ell=1}^{L}
		\bar\kappa_\ell^2\|\cE_\ell(\br_\ell)\|_F^2.
		\label{eq:compact-expected-bound}
	\end{align}
	The expectation is conditional on the fixed calibration.
\end{theorem}

\begin{proof}
	Fix $\br\in\mathcal R$ and order the adapted nodes
	topologically.
	Let $F^{(m)}_{\br}$ use untruncated proxy updates at
	nodes $1,\ldots,m$ and truncated updates at the remaining
	nodes, with all other parameters fixed.
	Then $F^{(L)}_{\br}=f^{\mathrm{proxy}}$,
	$F^{(0)}_{\br}=\widehat f_{\br}$, and
	\begin{equation}
		f^{\mathrm{proxy}}(x)-\widehat f_{\br}(x)
		=
		\sum_{\ell=1}^{L}
		\bigl(
		F^{(\ell)}_{\br}(x)
		-
		F^{(\ell-1)}_{\br}(x)
		\bigr).
		\label{eq:network-telescope}
	\end{equation}
	
	Each adjacent pair differs only at node $\ell$ and shares
	the same input feature map $z_\ell(x)$, satisfying
	$\|z_\ell(x)\|_2\le B_\ell(x)$.
	The convolution outputs differ by
	$\mathcal A_\ell[\cE_\ell(\br_\ell)]z_\ell(x)$.
	Using the downstream Lipschitz bound gives
	\begin{align}
		\|F^{(\ell)}_{\br}(x)-F^{(\ell-1)}_{\br}(x)\|_2
		&\le
		D_\ell
		\|\mathcal A_\ell[\cE_\ell(\br_\ell)]z_\ell(x)\|_2
		\nonumber\\
		&\le
		D_\ell\beta_\ell B_\ell(x)
		\|\cE_\ell(\br_\ell)\|_F.
		\label{eq:single-node-bound}
	\end{align}
	Summing over layers proves
	Eq.~\eqref{eq:compact-output-bound}.
	Since the residual tensors are fixed after calibration,
	Minkowski's inequality yields
	\[
	\left(
	\mathbb E
	\|f^{\mathrm{proxy}}(X)-\widehat f_{\br}(X)\|_2^2
	\right)^{1/2}
	\le
	\sum_{\ell=1}^{L}
	\bar\kappa_\ell\|\cE_\ell(\br_\ell)\|_F.
	\]
	Squaring and applying Cauchy--Schwarz gives
	Eq.~\eqref{eq:compact-expected-bound}.
\end{proof}

\begin{corollary}[Uniform bound for the unweighted relative residual]
	\label{cor:unweighted}
	Let
	\begin{equation}
		C=
		\max_{1\le\ell\le L}
		\left\{
		\bar\kappa_\ell^2\max(Z_\ell,\varepsilon)
		\right\}.
		\label{eq:uniform-coefficient}
	\end{equation}
	Then $C$ is independent of $\br\in\mathcal R$, and
	\begin{equation}
		\mathbb E
		\|f^{\mathrm{proxy}}(X)-\widehat f_{\br}(X)\|_2^2
		\le
		LC\sum_{\ell=1}^{L}d_\ell(\br_\ell).
		\label{eq:unweighted-output-bound}
	\end{equation}
\end{corollary}

\begin{proof}
	By Eq.~\eqref{eq:relative-residual},
	\[
	\|\cE_\ell(\br_\ell)\|_F^2
	=
	\max(Z_\ell,\varepsilon)d_\ell(\br_\ell).
	\]
	Consequently,
	\begin{align*}
		L\sum_{\ell=1}^{L}
		\bar\kappa_\ell^2\|\cE_\ell(\br_\ell)\|_F^2
		&=
		L\sum_{\ell=1}^{L}
		\bar\kappa_\ell^2
		\max(Z_\ell,\varepsilon)d_\ell(\br_\ell)\\
		&\le
		LC\sum_{\ell=1}^{L}d_\ell(\br_\ell).
	\end{align*}
	Combining this with
	Eq.~\eqref{eq:compact-expected-bound}
	proves the result.
\end{proof}

For fixed calibration and candidate sets,
Eq.~\eqref{eq:unweighted-output-bound} shows that minimizing
the sum of normalized residuals also minimizes the stated
upper bound on the expected squared output discrepancy
between the proxy network and its Tucker approximation.
This provides theoretical justification for the proposed
allocation objective.
The effect of residual normalization is further evaluated
in Section~\ref{subsec:normalization-ablation}.

\subsection{Budget-constrained rank allocation and final training}
\label{subsec:allocation-score}
\label{subsec:budget-allocation}
\label{subsec:final-training}

We construct candidates around a reference allocation with channel ratio $\rho\in(0,1]$. Its layerwise anchor ranks are
\begin{equation}
r^0_{\mathrm{out},\ell}
=
\max\left(1,\left\lceil\rho C_{\mathrm{out},\ell}\right\rceil\right),
\qquad
r^0_{\mathrm{in},\ell}
=
\max\left(1,\left\lceil\rho C_{\mathrm{in},\ell}\right\rceil\right).
\label{eq:anchor-ranks}
\end{equation}
Candidate rank pairs are formed by scaling the two anchors independently with a finite set of positive multipliers $\mathcal M$ and mapping the results to valid integer ranks. The candidate set includes the anchor pair and assigns positive ranks to every selected layer. We set the shared budget to the cost of the reference allocation,
\[
B_{\mathrm{adp}}=\sum_{\ell=1}^{L}p_\ell(\br_\ell^0),
\qquad
\br_\ell^0=(r^0_{\mathrm{out},\ell},r^0_{\mathrm{in},\ell}).
\]
This construction compares the adaptive and fixed-ratio allocations under the same ceiling. The values of $\rho$, $\mathcal M$, and the resulting budget are given in Section~\ref{subsec:calibration-settings}.

Using the scores in Eq.~\eqref{eq:allocation-score}, we associate each candidate with its exact parameter cost. For candidate $q$ in layer $\ell$, set $p_{\ell,q}=p_\ell(\br_{\ell,q})$ using Eq.~\eqref{eq:tucker-cost}. A candidate is Pareto-dominated if another candidate in the same layer has no greater cost and no greater score, with at least one strict inequality. Dominated candidates are removed, and one representative is retained for identical cost--score pairs. Reindex the retained candidates as $q=1,\ldots,Q_\ell$. The discrete allocation problem is
\begin{equation}
\min_{q_1,\ldots,q_L}
\sum_{\ell=1}^{L}S_{\ell,q_\ell}
\quad
\text{s.t.}
\quad
\sum_{\ell=1}^{L}p_{\ell,q_\ell}\le B_{\mathrm{adp}}.
\label{eq:mckp}
\end{equation}
Selecting one candidate per layer under a shared budget gives a multiple-choice knapsack problem. Integer parameter costs permit an exact budget-indexed dynamic program. The anchor configuration is feasible before pruning because its cost defines $B_{\mathrm{adp}}$; dominance pruning preserves feasibility even if an anchor is removed.

\begin{proposition}[Exact budgeted allocation]
	\label{prop:dp}
	Assume that $p_{\ell,q}$ and $B_{\mathrm{adp}}$ are nonnegative
	integers and that the finite candidate table admits a feasible allocation.
	For integer $b\in\{0,\ldots,B_{\mathrm{adp}}\}$, let $V(\ell,b)$
	be the minimum cumulative score obtained by selecting exactly one
	candidate from each of the first $\ell$ layers, with total cost
	at most $b$.
	Set $V(0,b)=0$, assign $+\infty$ to infeasible states, and define
	the minimum over an empty set as $+\infty$.
	Then, for $\ell=1,\ldots,L$,
	\begin{equation}
		V(\ell,b)
		=
		\min_{\substack{1\le q\le Q_\ell\\p_{\ell,q}\le b}}
		\left\{
		V(\ell-1,b-p_{\ell,q})+S_{\ell,q}
		\right\}.
		\label{eq:dp}
	\end{equation}
	Recording a minimizing candidate at each finite state with $\ell\ge1$
	and backtracking from $V(L,B_{\mathrm{adp}})$ returns a globally
	optimal solution of Eq.~\eqref{eq:mckp} over the retained discrete
	candidate sets.
\end{proposition}

\begin{proof}
	Any feasible allocation for the first $\ell$ layers selects a
	candidate $q$ at layer $\ell$, leaving a budget of
	$b-p_{\ell,q}$ for the preceding layers.
	Its cumulative score is therefore at least
	$V(\ell-1,b-p_{\ell,q})+S_{\ell,q}$.
	Conversely, whenever this value is finite, an optimal allocation
	for the preceding layers combined with candidate $q$ attains it.
	Minimizing over all candidates with $p_{\ell,q}\le b$ proves
	Eq.~\eqref{eq:dp}.
	
	Starting from $V(0,b)=0$, induction on $\ell$ establishes
	optimality for every state.
	Backtracking through the recorded minimizing candidates recovers
	an allocation attaining $V(L,B_{\mathrm{adp}})$.
	Pareto pruning preserves the optimal objective value because
	any removed candidate can be replaced by a retained candidate
	from the same layer with no greater cost or score.
\end{proof}

The dense dynamic program has pseudo-polynomial time complexity
\[
O\!\left(B_{\mathrm{adp}}\sum_{\ell=1}^{L}Q_\ell\right).
\]
It requires $O(B_{\mathrm{adp}})$ memory for objective values
and $O(LB_{\mathrm{adp}})$ for backtracking.
Our implementation uses sparse integer-cost states and prunes
dominated cost--score states after each layer, preserving the
optimal objective value.

\paragraph{Final PEFT training}

After rank allocation, the calibration model is discarded,
and final training starts from the original pretrained checkpoint.
Tucker adapters are initialized with Xavier-uniform channel factors
and zero-initialized cores, ensuring zero initial updates.
The adapters and common trainable components
(Section~\ref{subsec:network-training}) are then jointly optimized,
while the adapted pretrained kernels remain frozen.
The selected ranks remain fixed throughout training,
whereas the factor subspaces are learned independently of
the HOSVD bases.
The same initialization scheme is used for residual-guided
and fixed-ratio Tucker.

Algorithm~\ref{alg:method} summarizes the complete procedure.
Calibration and rank allocation use only training data,
while checkpoint and decision-threshold selection rely
exclusively on validation data.

\begin{algorithm}[t]
	\caption{Residual-guided Tucker rank allocation and training}
	\label{alg:method}
	\small
	\textbf{Input:} Pretrained checkpoint, training and validation
	partitions, candidate rank sets, and adapter budget
	$B_{\mathrm{adp}}$.\par
	\textbf{Output:} Selected ranks, trained adapters and common
	components, and a decision threshold.
	\begin{enumerate}[leftmargin=1.7em,itemsep=0.15em,topsep=0.4em]
		\item Sample a class-balanced calibration subset from the training
		partition and warm up the classification head with the encoder frozen.
		
		\item With all other parameters fixed, optimize the selected full
		kernels on the calibration subset to obtain
		$\Delta\cW_\ell^{\mathrm{proxy}}$.
		
		\item Evaluate candidate parameter costs using
		Eq.~\eqref{eq:tucker-cost} and normalized HOSVD residual scores
		using Eq.~\eqref{eq:allocation-score}.
		Remove duplicate and dominated candidates.
		
		\item Solve Eq.~\eqref{eq:mckp} by dynamic programming.
		\begin{enumerate}[label=(\alph*),leftmargin=1.7em,
			itemsep=0.1em,topsep=0.1em]
			\item Initialize $V(0,b)=0$ for all integer budgets
			$b=0,\ldots,B_{\mathrm{adp}}$.
			Set all states with $\ell\geq 1$ to $+\infty$.
			
			\item For $\ell=1,\ldots,L$, update $V(\ell,b)$
			using Eq.~\eqref{eq:dp} for each budget $b$.
			At each finite state, store a minimizing candidate
			index as $A(\ell,b)$.
			
			\item Set $b=B_{\mathrm{adp}}$ and backtrack
			for $\ell=L,\ldots,1$.
			Recover $q_\ell^*=A(\ell,b)$ and the selected
			rank pair $\br_\ell^*=\br_{\ell,q_\ell^*}$,
			then update $b\leftarrow b-p_{\ell,q_\ell^*}$.
		\end{enumerate}
		
		\item Discard the calibration model and reload the original
		pretrained checkpoint. Initialize adapters at the selected ranks
		with random channel factors and zero cores.
		
		\item Train the adapters and common components on the training
		partition, keeping the selected ranks and adapted pretrained
		kernels fixed. Select the checkpoint by validation AUC and
		determine its decision threshold on the validation set.
	\end{enumerate}
\end{algorithm}

\section{Experimental Setup}
\label{sec:experiments}

\subsection{Pretraining dataset and downstream cohort}

\paragraph{Hippocampal segmentation dataset}
Source pretraining uses the EADC--ADNI dataset, where EADC denotes
the European Alzheimer's Disease Consortium. The dataset contains
$135$ T1-weighted MRI scans with manual hippocampal annotations
based on the Harmonized Hippocampal Protocol (HarP) \cite{boccardi2015training}.
After excluding five scans with inconsistent annotations, the
remaining $130$ volumes are divided into $90$ training and
$40$ validation volumes. All downstream methods use the same
pretrained segmentation checkpoint. The network architecture
and pretraining protocol are described in
Section~\ref{subsec:network-training}.

\paragraph{MCI conversion cohort}
The downstream task is to predict conversion from MCI to AD
within $36$ months using baseline T1-weighted MRI.
Imaging data and longitudinal diagnoses are obtained from the
Alzheimer's Disease Neuroimaging Initiative (ADNI), a multicenter
longitudinal study (\url{https://adni.loni.usc.edu})
\cite{jack2008adni,adniDataAccess}.

Each participant contributes one baseline T1-weighted MRI scan.
Participants diagnosed with AD within $36$ months are classified
as progressive MCI (pMCI), the positive class, whereas stable
MCI (sMCI) is defined by at least $36$ months of follow-up
without conversion to AD.
The resulting cohort comprises $569$ participants
($305$ sMCI and $264$ pMCI), including $248$ from ADNI1,
$233$ from ADNI2, and $88$ from ADNI3.
Baseline characteristics are summarized in
Table~\ref{tab:cohort-characteristics}.

\begin{table}[t]
	\centering
	\small
	\caption{Baseline characteristics of the ADNI cohort.}
	\label{tab:cohort-characteristics}
	\renewcommand{\arraystretch}{1.15}
	\begin{tabular*}{\linewidth}{@{\extracolsep{\fill}}lccc@{}}
		\toprule
		Characteristic
		& \shortstack{sMCI\\$n=305$}
		& \shortstack{pMCI\\$n=264$}
		& \shortstack{Total\\$n=569$}\\
		\midrule
		Age (years), mean $\pm$ SD\textsuperscript{a}
		& $71.69 \pm 7.63$
		& $73.93 \pm 7.07$
		& $72.73 \pm 7.45$\\
		Female, $n$ (\%)
		& 120 (39.34) & 105 (39.77) & 225 (39.54)\\
		Male, $n$ (\%)
		& 185 (60.66) & 159 (60.23) & 344 (60.46)\\
		\bottomrule
	\end{tabular*}
	
	\smallskip
	\begin{minipage}{\linewidth}
		\footnotesize
		\textsuperscript{a}Age data are available for
		304 sMCI and 264 pMCI participants.
		SD denotes standard deviation.
	\end{minipage}
\end{table}

\subsection{Preprocessing and data split}

Both pretraining and downstream MRI scans undergo skull stripping,
affine registration to the Montreal Neurological Institute 152
(MNI152) standard brain template, and resampling to $1\ \mathrm{mm}$
isotropic resolution on a $182\times218\times182$ grid.
For segmentation pretraining, intensities are normalized using
the mean and standard deviation of positive-intensity voxels,
followed by random cropping to $128\times128\times128$ voxels.
Random flipping and intensity augmentation are applied during
training. For downstream classification, volumes are center-cropped
to $128\times128\times128$ voxels, with zero padding where necessary.
Intensities are then normalized using the mean and standard
deviation of nonzero voxels, with $10^{-8}$ added to the
denominator while preserving the zero-valued background.

The dataset is divided into training, validation, and test sets
at a ratio of $60{:}20{:}20$.
The training set contains $341$ participants
($183$ sMCI and $158$ pMCI), while the validation and test
sets each contain $114$ participants
($61$ sMCI and $53$ pMCI).
Calibration and model training use the training set,
while checkpoint and decision-threshold selection use
the validation set. The test set is reserved for final evaluation.

\subsection{Network and training protocol}
\label{subsec:network-training}

The downstream model reuses the encoder and bottleneck of a
3D residual U-Net pretrained for hippocampal segmentation
\cite{ronneberger2015unet,he2016resnet}.
The backbone processes single-channel $128\times128\times128$
volumes with channel widths of $32$, $64$, $128$, $256$, and $512$.
It consists of seven residual double-convolution blocks,
each with two bias-free $3\times3\times3$ convolutions,
and five channel-changing shortcuts using $1\times1\times1$
convolutions followed by BatchNorm.
The segmentation decoder is replaced by global average pooling
and a two-layer classification head that maps 512 features
to two output logits through 128 hidden units.
The rectified linear unit (ReLU) activation is applied after the
hidden layer, with dropout ($p=0.2$)
before each fully connected layer.

All adapter methods target the $14$ main-path
$3\times3\times3$ convolutions, whose pretrained weights
remain frozen. The adapters are trained jointly with
the classification head, five shortcut convolutions,
and BatchNorm affine parameters, which together contribute
$250{,}082$ common trainable parameters.

\paragraph{Source pretraining}
The segmentation network is trained from scratch for $500$ epochs
using Adam with a Dice-based loss, a batch size of $2$,
an initial learning rate of $10^{-4}$, and weight decay
of $2\times10^{-5}$.

\paragraph{Downstream training}
All methods are trained for $100$ epochs using Adam with
a batch size of $2$ and a base learning rate of $10^{-4}$.
Class-weighted cross-entropy is used, with weights calculated
from the training set. Weight decay is set to $10^{-5}$
for weight tensors and zero for biases and BatchNorm affine
parameters. Training uses automatic mixed precision (AMP)
with dynamic gradient scaling, without data augmentation.

BatchNorm running statistics are frozen during PEFT,
whereas full fine-tuning updates all $42.69$ million model
parameters and BatchNorm running statistics.
Each method is evaluated over five random seeds
($42$, $3407$, $2024$, $2025$, and $2026$).
The checkpoint with the highest validation AUC is selected
from each run, subject to the method-specific rules in
Section~\ref{subsec:comparisons}.

\subsection{Calibration and rank-allocation settings}
\label{subsec:calibration-settings}

For each run, a class-balanced subset of $68$ training
participants ($34$ per class, approximately $20\%$ of the
training set) is used for calibration.
The classification head is first warmed up for $100$ steps
with a batch size of $2$, followed by two epochs of proxy
optimization with a batch size of $1$ ($T_{\mathrm p}=136$ steps).
Both stages minimize unweighted cross-entropy using Adam
with a learning rate of $10^{-4}$ and zero weight decay.
Head dropout is enabled during warm-up but disabled during
proxy optimization, while encoder BatchNorm running
statistics remain frozen throughout calibration.

Candidate ranks are constructed as described in
Section~\ref{subsec:budget-allocation}, using a reference
ratio $\rho=1/16$ and multipliers
\begin{equation}
	\mathcal M=\{0.5,0.75,1.0,1.5,2.0\}.
	\label{eq:rank-multipliers}
\end{equation}
Input- and output-channel anchor ranks are scaled independently,
rounded to the nearest integer (with half-integer ties rounded
upward), and clipped to the corresponding channel dimensions.
The adapter budget is set to $B_{\mathrm{adp}}=367{,}819$,
matching fixed-ratio Tucker.
Including the common trainable components, the total
trainable-parameter ceiling is approximately $0.618$ million
($1.45\%$ of the full model).
The numerical floor in Eq.~\eqref{eq:relative-residual}
is set to $\varepsilon=10^{-12}$.

Calibration and rank allocation are performed independently
for each random seed. Before final PEFT training, the seed
is reset to its run-specific value. Both residual-guided
and fixed-ratio Tucker use Xavier-uniform initialization
for channel factors and zero initialization for Tucker cores.

\subsection{Comparative methods and ablation controls}
\label{subsec:comparisons}

We compare residual-guided Tucker with full fine-tuning,
four matrix-based PEFT methods, and fixed-ratio Tucker.
The matrix-based baselines include LoRA \cite{hu2022lora},
which uses low-rank additive updates; AdaLoRA
\cite{zhang2023adalora}, which dynamically allocates ranks
based on component importance; PiSSA \cite{meng2024pissa},
which initializes low-rank factors from principal singular
components while retaining a frozen residual; and PoLAR
\cite{lion2025polar}, which uses direction factors and a
dense scale matrix with landing-field optimization.
Fixed-ratio Tucker uses the parameterization in
Eq.~\eqref{eq:tucker-update} with $\rho=1/16$,
providing a direct comparison under the same
adapter-parameter budget.

All matrix-based methods unfold each selected kernel into
a $C_{\mathrm{out}}\times(C_{\mathrm{in}}K)$ matrix.
LoRA, PiSSA, and PoLAR use rank $4$ to obtain parameter
counts comparable to those of the Tucker methods.
All adapter methods target the same convolutional kernels
and optimize the common trainable components defined in
Section~\ref{subsec:network-training}.
They follow the same training protocol without
method-specific hyperparameter search.

AdaLoRA uses a parameter-cost budget based on fixed-ratio
Tucker and retains an importance-ranked subset of components,
allowing a small budget overshoot.
Its reported parameter count corresponds to the final compact
allocation, although the training model and optimizer retain
their initial tensor shapes.
Checkpoint selection is restricted to epochs after rank
allocation is finalized.
Further implementation details are provided in
Appendix~\ref{app:baseline-implementation}.

Ablation controls include head-only tuning and common-only
tuning, the latter optimizing the classification head,
channel-changing shortcuts, and BatchNorm affine parameters
without adapters.
To evaluate residual normalization, we also consider an
absolute-residual variant that uses squared Frobenius
reconstruction errors without normalization by proxy-update
energy. All other settings, including calibration, candidate
construction, parameter budget, initialization, and final
training, are identical to those of residual-guided Tucker.

\subsection{Evaluation metrics and efficiency protocol}
\label{subsec:evaluation-protocol}

Predictive performance is evaluated using the area under the
receiver operating characteristic curve (AUC), accuracy (ACC),
F1 score, sensitivity (SEN), and specificity (SPE).
Results are reported as the mean and sample standard deviation
across five random seeds on the same test cohort.
For threshold-dependent metrics, the decision threshold
maximizes validation specificity subject to
\begin{equation}
	\mathrm{SEN}_{\mathrm{val}}(t)\ge0.70.
	\label{eq:val-sen70}
\end{equation}
The selected threshold is applied unchanged to the test set.

Computational efficiency is evaluated using trainable parameter
counts, floating-point operation counts (FLOPs) for forward passes,
peak training memory, and test-set processing time.
Peak training memory and test time are measured
on an NVIDIA GeForce RTX 4090 D graphics processing unit (GPU)
with a batch size of $2$.
Cold-forward FLOPs include adapter reconstruction, whereas
warm-forward FLOPs reuse cached updates.
Test time includes data loading and transfer.
AdaLoRA is evaluated using its final compact configuration.
Calibration and initialization are excluded from these benchmarks,
so the reported memory reflects final-model training rather
than the complete adaptation process.

For residual-guided Tucker, preparation time is measured
separately, covering model setup, checkpoint loading,
calibration, candidate evaluation, rank allocation, and
final-model initialization. The measured interval also includes
an auxiliary squared-gradient computation not used for
rank allocation.

\section{Results}

\subsection{Predictive performance}

Residual-guided Tucker achieved the highest mean AUC among the
compared methods ($0.7706$; Table~\ref{tab:main} and
Figure~\ref{fig:parameter-auc}), exceeding
fixed-ratio Tucker by $4.67$ percentage points and full fine-tuning
by $2.76$ points. The comparison with fixed-ratio Tucker used
the same update parameterization and adapter-budget ceiling,
supporting the benefit of residual-guided rank allocation in the
evaluated setting. The proposed method also achieved a higher mean
AUC than the matrix-adapter baselines, whose mean AUCs ranged
from $0.7203$ to $0.7301$.
It used approximately $0.618$ million trainable
parameters, corresponding to $1.45\%$ of the full model.

For threshold-dependent metrics, residual-guided Tucker achieved
the highest mean ACC ($0.7035$), F1 score ($0.6763$), and
sensitivity ($0.6755$). Compared with fixed-ratio Tucker,
mean sensitivity increased from $0.5736$ to $0.6755$, while
mean specificity remained $0.7279$ at the reported precision.
Full fine-tuning achieved the highest mean specificity ($0.7639$),
reflecting a different sensitivity--specificity trade-off at the
validation-selected thresholds.

\begin{table}[t]
	\centering
	\small
	\caption{Predictive performance of the compared methods. Performance values are means across five seeds, with sample SDs in parentheses below. Bold denotes the highest mean in each performance column. For AdaLoRA, the parameter count refers to the final compact allocation.}
	\label{tab:main}
	\setlength{\tabcolsep}{3pt}\renewcommand{\arraystretch}{1.15}
\begin{tabular}{lcccccc}
			\toprule
			Method & \shortstack{Trainable\\(M)} & AUC & ACC & F1 & SEN & SPE\\
			\midrule
			
			Full fine-tuning
			& $42.691$
			& \meanSD{0.7430}{0.0099}
			& \meanSD{0.6684}{0.0218}
			& \meanSD{0.6100}{0.0229}
			& \meanSD{0.5585}{0.0413}
			& \bestSD{0.7639}{0.0539}\\
			
			LoRA
			& $0.646$
			& \meanSD{0.7301}{0.0243}
			& \meanSD{0.6579}{0.0206}
			& \meanSD{0.6194}{0.0403}
			& \meanSD{0.6038}{0.0789}
			& \meanSD{0.7049}{0.0544}\\
			
			AdaLoRA
			& $0.618$
			& \meanSD{0.7218}{0.0102}
			& \meanSD{0.6632}{0.0338}
			& \meanSD{0.6368}{0.0457}
			& \meanSD{0.6377}{0.0631}
			& \meanSD{0.6852}{0.0356}\\
			
			PiSSA
			& $0.646$
			& \meanSD{0.7218}{0.0186}
			& \meanSD{0.6579}{0.0317}
			& \meanSD{0.6022}{0.0309}
			& \meanSD{0.5585}{0.0605}
			& \meanSD{0.7443}{0.0849}\\
			
			PoLAR
			& $0.647$
			& \meanSD{0.7203}{0.0156}
			& \meanSD{0.6579}{0.0240}
			& \meanSD{0.6001}{0.0570}
			& \meanSD{0.5623}{0.1135}
			& \meanSD{0.7410}{0.0700}\\
			
			Fixed-ratio Tucker
			& $0.618$
			& \meanSD{0.7239}{0.0093}
			& \meanSD{0.6561}{0.0209}
			& \meanSD{0.6046}{0.0542}
			& \meanSD{0.5736}{0.0958}
			& \meanSD{0.7279}{0.0526}\\
			
			\shortstack[l]{Residual-guided\\Tucker (ours)}
 & $0.618$
 & \bestSD{0.7706}{0.0381}
 & \bestSD{0.7035}{0.0474}
 & \bestSD{0.6763}{0.0694}
 & \bestSD{0.6755}{0.1079}
 & \meanSD{0.7279}{0.0587}\\
\bottomrule
		\end{tabular}
\end{table}

\begin{figure}[t]
    \centering
    \includegraphics[width=\linewidth]{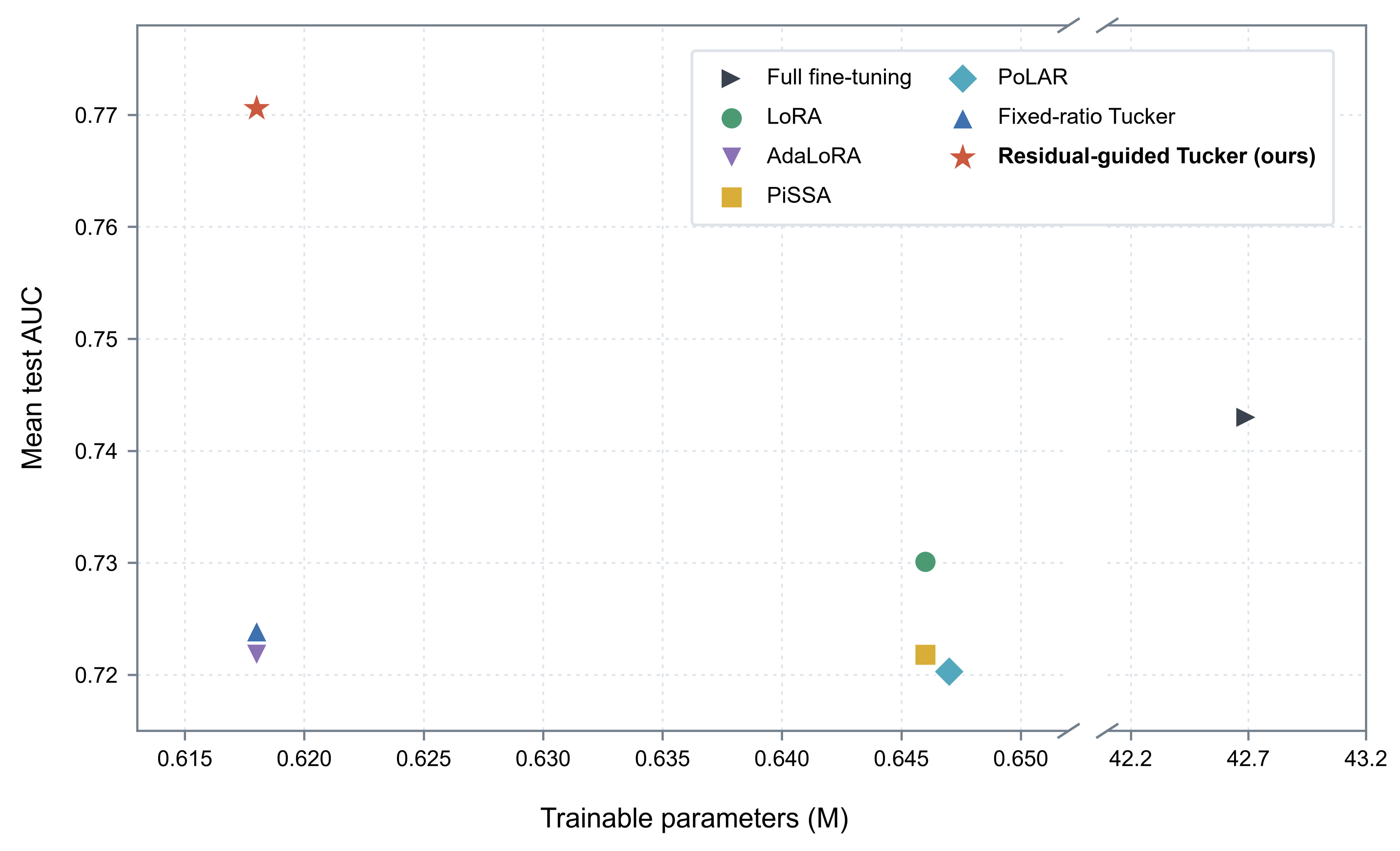}
    \caption{Mean test AUC versus trainable parameter count for the compared methods. The horizontal axis is broken to separate the PEFT methods from full fine-tuning; the two segments use different scales. M denotes millions of parameters. For AdaLoRA, the parameter count refers to the final compact allocation.}
    \label{fig:parameter-auc}
\end{figure}

\subsection{Ablation studies}
\label{subsec:normalization-ablation}

\paragraph{Contribution of trainable components}

Training the common components in addition to the classification
head increased mean AUC from $0.6820$ to $0.7250$
(Table~\ref{tab:ablation-controls}). Adding fixed-ratio Tucker
adapters yielded a mean AUC of $0.7239$, suggesting that additional
adapter capacity alone did not improve discrimination in this
setting. Residual-guided Tucker achieved a mean AUC of $0.7706$,
exceeding the common-components-only control by $4.56$ percentage
points. This improvement was accompanied by higher mean sensitivity
($0.6755$ versus $0.5321$) and lower mean specificity
($0.7279$ versus $0.7672$).

\paragraph{Effect of residual normalization}

We compared allocations based on normalized and absolute
reconstruction errors:
\begin{equation}
	S_{\ell,q}^{\mathrm{rel}}
	=\frac{\|\cE_{\ell,q}\|_F^2}{\max(Z_\ell,\varepsilon)},
	\qquad
	S_{\ell,q}^{\mathrm{abs}}=\|\cE_{\ell,q}\|_F^2.
	\label{eq:normalization-ablation}
\end{equation}
Both variants used the same adapter-budget ceiling and differed
only in the allocation score, with all other settings matched
as described in Section~\ref{subsec:comparisons}.

The normalized score achieved a mean AUC of $0.7706$, compared
with $0.7524$ for the absolute score, an increase of $1.82$
percentage points (Table~\ref{tab:normalization-ablation}).
Mean sensitivity increased from $0.6226$ to $0.6755$, while
mean specificity decreased from $0.7475$ to $0.7279$. 

\begin{table}[t]
	\centering
	\small
	\caption{Contribution of the trainable components. Values are means across five seeds, with sample SDs in parentheses below. Common denotes the classification head, channel-changing shortcuts, and BatchNorm affine parameters.}
	\label{tab:ablation-controls}
	\setlength{\tabcolsep}{3pt}\renewcommand{\arraystretch}{1.15}
\begin{tabular}{lcccccc}
			\toprule
			Configuration & \shortstack{Trainable\\(M)} & AUC & ACC & F1 & SEN & SPE\\
			\midrule
			
			Head-only
			& $0.066$
			& \meanSD{0.6820}{0.0107}
			& \meanSD{0.6351}{0.0229}
			& \meanSD{0.5520}{0.0128}
			& \meanSD{0.4830}{0.0103}
			& \meanSD{0.7672}{0.0484}\\
			
			Common only
			& $0.250$
			& \meanSD{0.7250}{0.0133}
			& \meanSD{0.6579}{0.0164}
			& \meanSD{0.5903}{0.0290}
			& \meanSD{0.5321}{0.0488}
			& \meanSD{0.7672}{0.0425}\\
			
			\shortstack[l]{Fixed-ratio Tucker\\+ common}
			& $0.618$
			& \meanSD{0.7239}{0.0093}
			& \meanSD{0.6561}{0.0209}
			& \meanSD{0.6046}{0.0542}
			& \meanSD{0.5736}{0.0958}
			& \meanSD{0.7279}{0.0526}\\
			
			\shortstack[l]{Residual-guided Tucker\\+ common}
 & $0.618$
 & \meanSD{0.7706}{0.0381}
 & \meanSD{0.7035}{0.0474}
 & \meanSD{0.6763}{0.0694}
 & \meanSD{0.6755}{0.1079}
 & \meanSD{0.7279}{0.0587}\\
			
			\bottomrule
		\end{tabular}
\end{table}

\begin{table}[t]
\centering
\small
\caption{Effect of residual normalization under a common adapter-budget ceiling. Values are means across five seeds, with sample SDs in parentheses below. Bold denotes the higher mean in each column.}
\label{tab:normalization-ablation}
\setlength{\tabcolsep}{3pt}\renewcommand{\arraystretch}{1.15}
\begin{tabular}{lccccc}
\toprule
Allocation score & AUC & ACC & F1 & SEN & SPE\\
\midrule
Absolute residual & \meanSD{0.7524}{0.0253} & \meanSD{0.6895}{0.0395} & \meanSD{0.6480}{0.0647} & \meanSD{0.6226}{0.0962} & \bestSD{0.7475}{0.0220}\\
\shortstack[l]{Normalized residual\\(proposed)} & \bestSD{0.7706}{0.0381} & \bestSD{0.7035}{0.0474} & \bestSD{0.6763}{0.0694} & \bestSD{0.6755}{0.1079} & \meanSD{0.7279}{0.0587}\\
\bottomrule
\end{tabular}
\end{table}

\subsection{Layerwise rank and parameter allocation}

Residual-guided Tucker redistributed adapter capacity across
layers and channel modes relative to fixed-ratio Tucker
(Figure~\ref{fig:allocation}, seed 42). Output-channel ranks
increased in the initial block and the first three downsampling
blocks, but decreased in both convolutions of the fourth
downsampling block and the second convolution of each bottleneck
block. Input-channel ranks increased in the second convolution
of the initial block, throughout the first three downsampling
blocks, and in the first bottleneck convolution. They decreased
in the second convolution of the fourth downsampling block and
the second convolution of each bottleneck block; the remaining
input-channel ranks were unchanged.

The two channel ranks did not necessarily change together.
For example, the first bottleneck convolution retained an output
rank of $32$ while its input rank increased from $32$ to $64$,
whereas the following convolution used output and input ranks
of $24$ and $16$, respectively. The corresponding parameter
allocation increased capacity in the earlier blocks and the
first bottleneck convolution while reducing it in several later
convolutions. Overall, the selected configuration used $99.93\%$
of the shared adapter budget, leaving $270$ parameters unused.
These patterns reflect proxy reconstruction quality and coupled
parameter costs, rather than direct measures of anatomical or
predictive importance.

\begin{figure}[t]
	\centering
	\includegraphics[width=\linewidth]{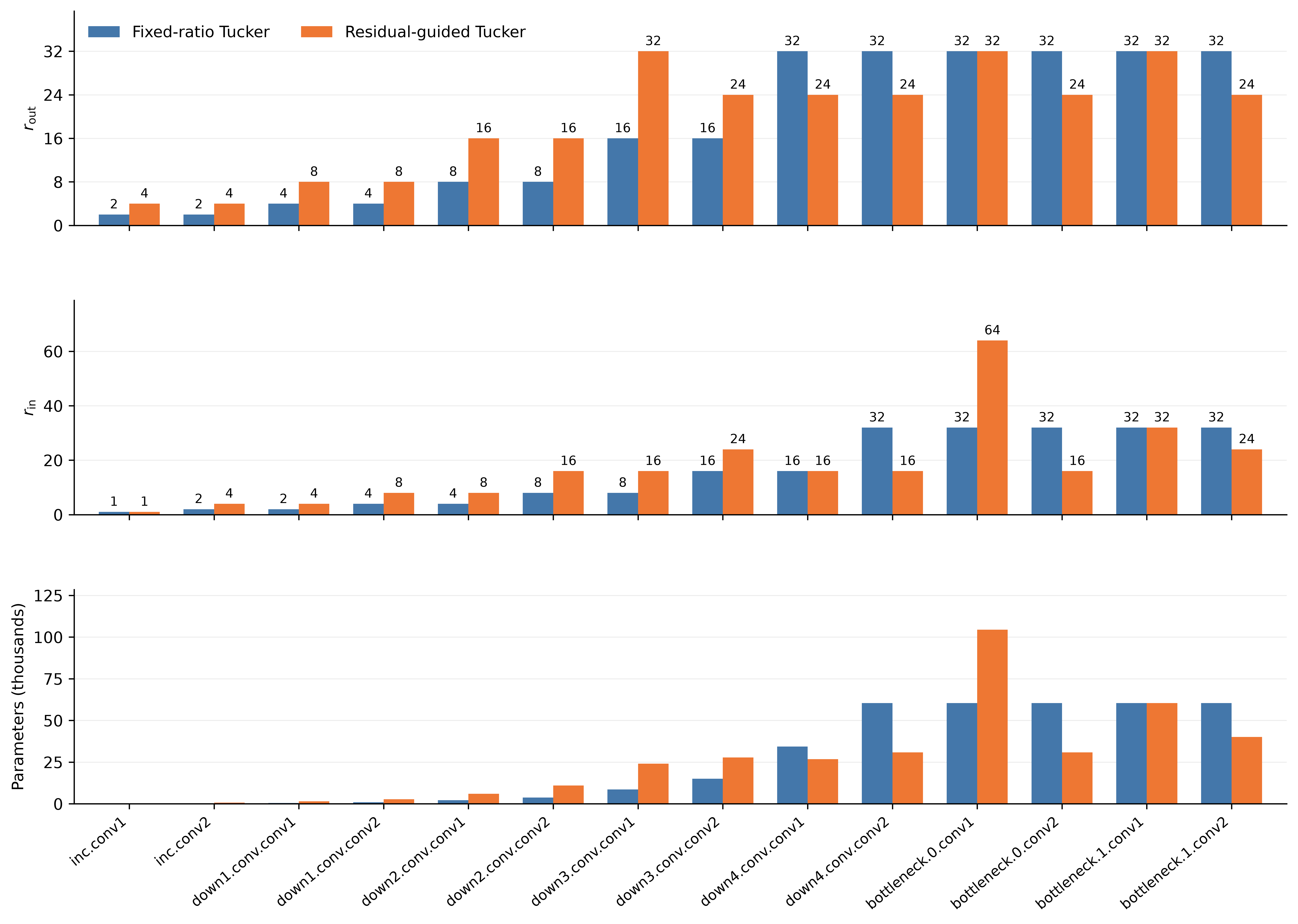}
	\caption{Layerwise channel ranks and adapter parameter counts for residual-guided Tucker and fixed-ratio Tucker (seed 42).}
	\label{fig:allocation}
\end{figure}

\subsection{Parameter and computational efficiency}

Residual-guided Tucker combined a small trainable parameter count
with reduced peak training memory (Table~\ref{tab:efficiency}).
For the evaluated configuration (seed 42), the adapters contained
$0.368$ million parameters, bringing the total trainable count to
$0.618$ million, or $1.45\%$ of the full model.
Peak allocated memory during final training was $2.235$ gibibytes (GiB),
a $17.31\%$ reduction relative to full fine-tuning ($2.703$ GiB).

These parameter and memory savings were achieved with forward
computational costs comparable to those of full fine-tuning.
A cold forward pass required $316.203$ billion floating-point
operations (GFLOPs), compared with
$313.801$ GFLOPs for full fine-tuning, with the additional cost
arising from adapter reconstruction. Caching the reconstructed
updates reduced the warm-forward cost to $313.801$ GFLOPs,
equal to that of full fine-tuning.
Test-set processing times were similar across methods.
For residual-guided Tucker, preparation, including model setup,
calibration, rank allocation, and final-model initialization,
required approximately $26$ s in the seed-42 run, accounting
for $1.5\%$ of the total preparation and training time.

\begin{table}[t]
	\centering
	\small
	\caption{Parameter and computational efficiency.
		FLOPs are measured for a cold forward pass, and memory denotes
		peak allocated memory during final training.
		Trainable parameter percentages are relative to full fine-tuning.
		Test-set processing time includes data loading and transfer.
		AdaLoRA is evaluated at its final compact allocation.
		Measurements exclude calibration and initialization.}
	\label{tab:efficiency}
	\setlength{\tabcolsep}{1.5pt}\renewcommand{\arraystretch}{1.15}
\begin{tabular}{lcccccc}
			\toprule
			Method
			& \shortstack{Adapter\\(M)}
			& \shortstack{Trainable\\(M)}
			& \shortstack{Trainable\\(\%)}
			& \shortstack{FLOPs\\(G)}
			& \shortstack{Peak memory\\(GiB)}
			& \shortstack{Test time\\(s)}\\
			\midrule
			
			Full fine-tuning
			& --
			& $42.691$
			& $100.00$
			& $313.801$
			& $2.703$
			& $4.955$\\
			
			LoRA 
			& $0.396$
			& $0.646$
			& $1.51$
			& $314.141$
			& $2.233$
			& $4.943$\\
			
			AdaLoRA
			& $0.368$
			& $0.618$
			& $1.45$
			& $314.095$
			& $2.258$
			& $4.948$\\
			
			PiSSA 
			& $0.396$
			& $0.646$
			& $1.51$
			& $314.141$
			& $2.233$
			& $4.998$\\
			
			PoLAR 
			& $0.397$
			& $0.647$
			& $1.51$
			& $314.141$
			& $2.233$
			& $4.994$\\
			
			Fixed-ratio Tucker
			& $0.368$
			& $0.618$
			& $1.45$
			& $316.543$
			& $2.235$
			& $5.134$\\
			
			\shortstack[l]{Residual-guided\\Tucker (ours)} & $0.368$ & $0.618$ & $1.45$ & $316.203$ & $2.235$ & $5.049$\\
\bottomrule
		\end{tabular}
\end{table}

\section{Discussion}

\subsection{Task-conditioned allocation of adapter capacity}

The comparison with fixed-ratio Tucker suggests that the
distribution of adapter capacity across layers affects adaptation
performance under a fixed adapter budget. The component ablations
provide complementary evidence: adding fixed-ratio Tucker adapters
did not improve mean AUC over tuning the common components alone,
whereas residual-guided Tucker achieved a higher mean AUC than
both configurations while sharing the fixed-ratio method's
adapter-budget ceiling. These findings support task-conditioned
capacity allocation in the evaluated setting.
Normalization measures reconstruction error relative to each
layer's proxy-update energy, reducing the influence of differences
in update magnitude. The normalization ablation provides empirical
support for this relative scoring criterion.

Building on tensor adaptation and budgeted factorization
\cite{si2024flora,yin2022batude,ali2026rocket}, the method uses
task-specific weight displacements to guide joint selection
of input- and output-channel ranks while accounting for their
coupled parameter costs. This formulation links capacity allocation
to the structure of downstream updates while preserving the
spatial kernel modes and respecting a shared adapter budget.

\subsection{Theoretical rationale for rank allocation}

The analysis connects proxy-update approximation with rank
allocation under a shared parameter budget. For a fixed
calibrated proxy network, the perturbation bound relates
channel-mode reconstruction errors to changes in network
outputs. This connection gives reconstruction quality a
network-level interpretation: approximating a task-specific
update controls a source of output discrepancy in the proxy
network. It therefore motivates evaluating candidate ranks
by the update information they preserve, complementing
rank choices based solely on layer dimensions.

The optimization formulation accounts for the interaction
between the two channel ranks. Each rank pair determines
both approximation quality and parameter cost, including
the bilinear contribution of the Tucker core. Selecting ranks
jointly across layers allows these competing demands to be
resolved under a shared budget. Dynamic programming yields
a globally optimal solution for the specified candidate sets
and scores, ensuring that the allocation objective is minimized
without relying on a greedy sequence of local rank decisions.
This guarantee concerns the discrete approximation surrogate,
whose relevance to prediction is assessed through final training.

Final training retains the selected ranks while reinitializing
the adapters, so proxy approximation determines capacity without
restricting the learned updates to the proxy singular subspaces.
This separation allows the final adapters to learn representations
suited to the downstream objective within the selected configuration.
Normalization further expresses reconstruction errors relative
to each layer's proxy-update energy, placing layers on a common
relative scale. The perturbation analysis motivates residual-based
approximation, while the normalization ablation provides empirical
support for this particular scoring choice. Together, the analysis
and experiments connect the design of the allocation criterion,
its exact optimization, and its utility in downstream adaptation.

\subsection{Limitations and future directions}
\label{subsec:limitations}

The findings are based on the evaluated ADNI cohort and
experimental configuration, and further validation is needed
to establish their applicability to routine clinical populations.
Results across five seeds characterize training variability
within this setting, with prediction consistency remaining
an important consideration for further development.
Although the method reduces trainable parameters and peak
memory during final PEFT training, dense forward computation
remains comparable to that of full fine-tuning. Calibration
also contributes preparation time and temporary memory
requirements to the overall adaptation cost.

Future work will investigate whether MRI-based progression
estimates can complement clinical assessment of patients
with MCI. Independent, multicenter validation should examine
performance across scanners, acquisition protocols, and patient
populations, with attention to risk calibration, subgroup
performance, and prediction consistency.
Prospective studies could then assess whether integrating
model predictions into clinical workflows provides additional
value for follow-up planning or trial recruitment.
These evaluations would help identify appropriate operating
thresholds and clarify the role of the approach in clinical
decision support.

\section{Conclusion}

We introduced residual-guided Tucker for task-conditioned
allocation of adapter capacity in pretrained 3D networks.
The method uses normalized reconstruction residuals of
training-derived proxy updates to jointly select input- and
output-channel ranks under a fixed adapter-parameter budget.
A fixed-network perturbation analysis motivates residual-based
allocation, and dynamic programming yields an exact solution
for the specified candidate ranks and scores. On the ADNI dataset, residual-guided Tucker achieved a higher mean AUC than fixed-ratio Tucker for MCI conversion prediction under the same adapter-budget ceiling, with total
trainable parameters corresponding to $1.45\%$ of the full model.
These findings support task-conditioned rank allocation as
a promising approach to parameter-efficient transfer in
3D medical imaging and motivate further evaluation across
cohorts, architectures, and parameter budgets.

\section*{Acknowledgements}

Data collection and sharing for the Alzheimer's Disease Neuroimaging
Initiative (ADNI) are funded by the National Institute on Aging
(National Institutes of Health Grant U19~AG024904).
The grantee organization is the Northern California Institute for
Research and Education. In the past, ADNI has also received funding
from the National Institute of Biomedical Imaging and Bioengineering,
the Canadian Institutes of Health Research, and private sector
contributions through the Foundation for the National Institutes of
Health (FNIH), including generous contributions from the following:
AbbVie; Alzheimer's Association; Alzheimer's Drug Discovery Foundation;
Araclon Biotech; BioClinica, Inc.; Biogen; Bristol-Myers Squibb Company;
CereSpir, Inc.; Cogstate; Eisai Inc.; Elan Pharmaceuticals, Inc.;
Eli Lilly and Company; EuroImmun; F. Hoffmann-La Roche Ltd and its
affiliated company Genentech, Inc.; Fujirebio; GE Healthcare; IXICO Ltd.;
Janssen Alzheimer Immunotherapy Research \& Development, LLC.;
Johnson \& Johnson Pharmaceutical Research \& Development LLC.;
Lumosity; Lundbeck; Merck \& Co., Inc.; Meso Scale Diagnostics, LLC.;
NeuroRx Research; Neurotrack Technologies; Novartis Pharmaceuticals
Corporation; Pfizer Inc.; Piramal Imaging; Servier; Takeda Pharmaceutical
Company; and Transition Therapeutics.

\section*{CRediT authorship contribution statement}

\textbf{Guanghua He}: Conceptualization, Methodology, Writing -- original draft.
\textbf{Hancan Zhu}: Software, Investigation, Writing -- review \& editing.
\textbf{Gaohang Yu}: Supervision, Validation.
\textbf{An Zhang}: Validation, Formal analysis.

\section*{Declaration of competing interest}

The authors declare that they have no known financial conflicts of interest or personal relationships that could have influenced the research presented in this manuscript.

\section*{Data availability}

The imaging and clinical data used in this study were obtained
from the Alzheimer's Disease Neuroimaging Initiative (ADNI)
database and are accessible to approved researchers through
the LONI Image and Data Archive (IDA), subject to the ADNI
Data Use Agreement. Details on data access are available at
\url{https://adni.loni.usc.edu/data-samples/adni-data/}.

\appendix

\section{Proof of the channel-mode HOSVD residual bound}
\label{app:hosvd}

\begin{proof}[Proof of Lemma~\ref{lem:hosvd-bound}]
	Let $P_1=P_{\mathrm{out}}(r_{\mathrm{out}};\cX)$ and
	$P_2=P_{\mathrm{in}}(r_{\mathrm{in}};\cX)$.
	Adding and subtracting $\cX\times_1P_1$ gives
	\begin{align}
		\cX-\cX\times_1P_1\times_2P_2
		&=(\cX-\cX\times_1P_1)
		\nonumber\\
		&\quad+
		(\cX\times_1P_1-\cX\times_1P_1\times_2P_2).
		\label{eq:appendix-hosvd-decomp}
	\end{align}
	Since $P_1$ is an orthogonal projector, the two terms
	are orthogonal in the Frobenius inner product. Thus,
	\begin{align}
		\|\cX-\cX\times_1P_1\times_2P_2\|_F^2
		&=\|\cX-\cX\times_1P_1\|_F^2
		\nonumber\\
		&\quad+
		\|\cX\times_1P_1-\cX\times_1P_1\times_2P_2\|_F^2.
		\label{eq:appendix-hosvd-orthogonal}
	\end{align}
	Mode products on distinct modes commute, and orthogonal
	projections are nonexpansive. Therefore,
	\begin{equation}
		\|(\cX-\cX\times_2P_2)\times_1P_1\|_F
		\le
		\|\cX-\cX\times_2P_2\|_F.
		\label{eq:appendix-hosvd-nonexp}
	\end{equation}
	By the Eckart--Young theorem, the projection errors
	along the two channel modes satisfy
	\begin{align}
		\|\cX-\cX\times_1P_1\|_F^2
		&=\sum_{i>r_{\mathrm{out}}}\sigma_{\mathrm{out},i}^2,
		\label{eq:appendix-tail-out}\\
		\|\cX-\cX\times_2P_2\|_F^2
		&=\sum_{j>r_{\mathrm{in}}}\sigma_{\mathrm{in},j}^2.
		\label{eq:appendix-tail-in}
	\end{align}
	Combining these results yields
	Eq.~\eqref{eq:hosvd-tail-bound}.
\end{proof}

\section{Implementation of matrix-adapter baselines}
\label{app:baseline-implementation}

All four baselines adapt the same $14$ convolutional kernels,
unfolded as $W_0\in\R^{m\times n}$, where
$m=C_{\mathrm{out}}$ and $n=27C_{\mathrm{in}}$.
Each method retains $250{,}082$ common trainable parameters,
with an effective adapter scaling factor of $s=1$.

\paragraph{LoRA \cite{hu2022lora}}
LoRA uses $\Delta W=sBA$ with rank $r=4$ and
$\alpha_{\mathrm{LoRA}}=4$, where
$s=\alpha_{\mathrm{LoRA}}/r$.
Kaiming-uniform initialization of $A$ and zero initialization
of $B$ ensure zero initial updates.
Dropout is applied only in the classification head.
The adapter and total trainable parameter counts are
$396{,}396$ and $646{,}478$, respectively.

\paragraph{AdaLoRA \cite{zhang2023adalora}}
AdaLoRA uses $\Delta W=sP\operatorname{diag}(e)Q$ with an
initial rank of $6$. The factors $P$ and $Q$ are initialized
from a Gaussian distribution with standard deviation $0.02$,
while $e=0$.
Component importance is estimated from moving averages of
absolute parameter--gradient products and their absolute
deviations, both with decay $0.85$.
Each component score combines its singular-value importance
with the mean importance of the corresponding factor
column and row.

We adapt AdaLoRA's rank-count budget to a parameter-cost
budget. For $T$ training steps, let
$t_w=\lfloor0.1T\rfloor$, $t_e=T-\lfloor0.2T\rfloor$,
$b_0$ denote the initial adapter parameter count, and
$b_*=367{,}819$.
The budget schedule is
\begin{equation}
	b(t)=
	\begin{cases}
		b_0, & t\le t_w,\\
		b_*+(b_0-b_*)
		\left(1-\dfrac{t-t_w}{t_e-t_w}\right)^3,
		& t_w<t<t_e,\\
		b_*, & t\ge t_e.
	\end{cases}
\end{equation}
Masks are updated every $100$ steps after warm-up and
finalized at $t_e$.
Each component costs $m+n+1$ parameters.
The importance-ranked prefix closest to the target budget
is retained, with ties favoring lower cost.
The final compact adapter contains $368{,}162$ parameters,
exceeding the target by $343$ ($0.093\%$), and has
$618{,}244$ total trainable parameters.

Masked components may regrow before $t_e$ but remain masked
thereafter.
After $t_e$, the training objective includes an orthogonality
penalty of $0.5$ times the mean layerwise value of
$\|P^\top P-I\|_F^2+\|QQ^\top-I\|_F^2$,
computed over retained components.
Checkpoint selection by validation AUC begins after
rank allocation is finalized.
The saved factors are compact, while the training model
and optimizer retain their initial tensor shapes.

\paragraph{PiSSA}
PiSSA uses rank $r=4$ and initializes its factors using
an exact reduced singular value decomposition (SVD)
in 32-bit floating-point precision (FP32) on the CPU.
For the leading singular triplets $U_r,\Sigma_r,V_r$,
\begin{equation}
	B_0=U_r(\Sigma_r/s)^{1/2},\quad
	A_0=(\Sigma_r/s)^{1/2}V_r^\top,\quad
	W_{\mathrm{res}}=W_0-sB_0A_0.
\end{equation}
The frozen residual and trainable factors reconstruct
$W_0$ at initialization up to numerical precision.
Parameter counts are identical to those of LoRA.

\paragraph{PoLAR}
PoLAR \cite{lion2025polar} uses
$\Delta W=sX\Theta Y^\top$ with rank $r=4$,
where $X\in\R^{m\times r}$,
$Y\in\R^{n\times r}$, and
$\Theta\in\R^{r\times r}$.
The factors $X$ and $Y$ are initialized with orthonormal columns
obtained by QR decomposition of Gaussian random matrices.
The matrix $\Theta$ is initialized to zero.
After AMP gradient unscaling, the gradient $G_X$
is replaced by the landing field
\begin{equation}
	\Gamma_X=
	\tfrac12\left[
	G_X(X^\top X)-X(G_X^\top X)
	\right]
	+4\lambda_{\mathrm P}X(X^\top X-I),
	\qquad
	\lambda_{\mathrm P}=5\times10^{-3}.
\end{equation}
The same update rule applies to $Y$, while $\Theta$
uses its Euclidean gradient.
Adam uses the common learning rate and weight decay
$10^{-5}$ for all factors, without per-step QR retraction.
The adapter and total trainable parameter counts are
$396{,}620$ and $646{,}702$, respectively.

\end{document}